\documentclass[]{article}
\usepackage{proceed2e}
\pdfoutput=1
\usepackage{natbib}
\usepackage{amsmath, amsthm, amssymb}
\usepackage{graphicx,subfigure}
\usepackage{color}
\usepackage{algorithm,algorithmic}
\title{Pitman-Yor Diffusion Trees}

\author{David A. Knowles, Zoubin Ghahramani \\ University of Cambridge}
\begin{document}

\maketitle

\begin{abstract}
We introduce the Pitman Yor Diffusion Tree (PYDT) for hierarchical clustering, a generalization of the Dirichlet Diffusion Tree~\citep{Neal2001b} which removes the restriction to binary branching structure. The generative process is described and shown to result in an exchangeable distribution over data points. We prove some theoretical properties of the model and then present two inference methods: a collapsed MCMC sampler which allows us to model uncertainty over tree structures, and a computationally efficient greedy Bayesian EM search algorithm. Both algorithms use message passing on the tree structure. The utility of the model and algorithms is demonstrated on synthetic and real world data, both continuous and binary. 
\end{abstract}

Tree structures play an important role in machine learning and statistics.  Learning a tree structure over data points gives a straightforward picture of how objects of interest are related.  Trees are easily interpreted and intuitive to understand.  Sometimes we may know that there is a true underlying hierarchy: for example species in the tree of life or duplicates of genes in the human genome, known as paralogs.  Typical mixture models, such as Dirichlet Process mixture models, have independent parameters for each component.  We might expect for example that certain clusters are similar, for example are sub-groups of some large group.  By learning this hierarchical similarity structure, the model can share statistical strength between components to make better estimates of parameters using less data. 

Classical hierarchical clustering algorithms employ a bottom up ``agglomerative'' approach~\citep{Duda2000} which hides the statistical assumptions being made. \citet{Heller2005} use a principled probabilistic model in lieu of a distance metric but simply view the hierarchy as a tree consistent mixture over partitions of the data. If instead a full generative model for both the tree structure and the data is used~\citep{Williams2000,Neal2003a,Teh2008,bleincrp} Bayesian inference machinery can be used to compute posterior distributions over the tree structures themselves. Such models can also be used to learn hierarchies over latent variables~\citep{ihfrm}. 

Both heuristic and generative probabilistic approaches to learning hierarchies have focused on learning binary trees. Although computationally convenient this restriction may be undesirable: where appropriate, arbitrary trees provide a more interpretable, clean summary of the data. Some recent work has aimed to address this.~\cite{BluTehHel2010a} extend~\cite{Heller2005} by removing the restriction to binary trees. However, as for~\citet{Heller2005} the lack of a generative process prohibits modeling uncertainty over tree structures. \cite{Williams2000} allows nonbinary trees by having each node independently pick a parent in the layer above, but requires one to pre-specify the number of layers and number of nodes in each layer. \cite{bleincrp} use the nested Chinese restaurant process to define probability distributions over tree structures. Each data point is drawn from a mixture over the parameters on the path from the root to the data point, which is appropriate for mixed membership models but not standard clustering. An alternative to the PYDT to obtain unbounded trees is given by \cite{adams-ghahramani-jordan-2010a}. They use a nested stick-breaking process to construct the tree, which is then endowed with a diffusion process. Data live at internal nodes of the tree, rather than at leaves as in the PYDT. 

We introduce the Pitman Yor Diffusion Tree (PYDT), a generalization of the Dirichlet Diffusion Tree~\citep{Neal2001b} to trees with arbitrary branching structure. While allowing atoms in the divergence function of the DDT can in principle be used to obtain multifurcating branch points~\citep{Neal2003a}, our solution is both more flexible and more mathematically and computationally tractable. An interesting property of the PYDT is that the implied distribution over tree structures corresponds to the multifurcating Gibbs fragmentation tree~\citep{gibbstree}, known to be the most general process generating exchangeable and consistent trees (here consistency can be understood as coherence under marginalization of subtrees). 

Our contributions are as follows. In Section~\ref{sec:genprocess} we describe the generative process corresponding to the PYDT. In Section~\ref{sec:probtree} we derive the probability of a tree, and in Section~\ref{sec:theory} show some important properties of the process. Section~\ref{sec:model} describes our hierarchical clustering models utilising the PYDT. In Section~\ref{sec:inference} we present an MCMC sampler and a greedy EM algorithm, which we developped for the DDT in~\cite{knowles2011icml}.  We present results demonstrating the utility of the PYDT in Section~\ref{sec:results}. In the supplementary material we describe how to sample from the PYDT. 

\section{Generative process} \label{sec:genprocess}

We will describe the generative process for the data in terms of a diffusion process in fictitious ``time'' on the unit interval. The observed data points (or latent variables) correspond to the locations of the diffusion process at time $t=1$. The first datapoint starts at time 0 at the origin in a $D$-dimensional Euclidean space and follows a Brownian motion with variance $\sigma^2$ until time 1.  If datapoint $1$ is at position $x_1(t)$ at time $t$, the point will reach position $x_1(t+dt) \sim \textrm{N}(x_1(t),\sigma^2 I dt)$ at time $t + dt$.  It can easily be shown that $x_1(t) \sim \textrm{Normal}(0, \sigma^2 I t)$.  The second point $x_2$ in the dataset also starts at the origin and initially follows the path of $x_1$.  The path of $x_2$ will diverge from that of $x_1$ at some time $T_d$ after which $x_2$ follows a Brownian motion independent of $x_1(t)$ until $t=1$, with $x_i(1)$ being the $i$-th data point. The probability of diverging in an interval $[t+dt]$ is determined by a ``divergence function'' $a(t)$ (see Equation~\ref{eq:at} below) which is analogous to the hazard function in survival analysis. 

The generative process for datapoint $i$ is as follows. Initially $x_i(t)$ follows the path of the previous datapoints. If at time $t$ the path of $x_i(t)$ has not diverged, it will diverge in the next infinitesimal time interval $[t,t+dt]$ with probability
\begin{align}
\frac{a(t)\Gamma(m-\beta)dt}{\Gamma(m+1+\alpha)} \label{eq:at}
\end{align}
where $m$ is the number of datapoints that have previously followed the current path and $0\leq \beta \leq 1, \alpha \geq -2\beta$ are parameters of the model. In the special case of integer $\alpha \in \mathbb{N}$ and $\beta = 0$ the probability of divergence reduces to $a(t)dt/[m(m+1)\dots(m+\alpha-1)(m+\alpha)]$. For example for $\alpha=1$ this gives $a(t)dt/[m(m+1)]$, and for $\alpha=\beta=0$ the DDT expression $a(t)dt/m$ is recovered. If $x_i$ does not diverge before reaching a previous branching point, it may either follow one of the previous branches, or diverge at the branch point (adding one to the degree of this node in the tree). The probability of following one of the existing branches $k$ is 
\begin{align}
\frac{b_k - \beta}{m+\alpha} \label{eq:oldblock}
\end{align}
where $b_k$ is the number of samples which previously took branch $k$ and $m$ is the total number of samples through this branch point so far. The probability of diverging at the branch point and creating a new branch is
\begin{align}
\frac{\alpha+\beta K}{m+\alpha} \label{eq:newblock}
\end{align}
where $K$ is the number of branches from this branch point. By summing Equation~\ref{eq:oldblock} over $k=\{1,\dots,K\}$ with Equation~\ref{eq:newblock} we get 1 as required. This reinforcement scheme is analogous to the Pitman Yor process~\citep{Teh2006b} version of the Chinese restaurant process~\citep{crp}. For the single data point $x_i(t)$ this process is iterated down the tree until divergence, after which $x_i(t)$ performs independent Brownian motion until time $t=1$. The $i$-th observed data point is given by the location of this Brownian motion at $t=1$, i.e. $x_i(1)$. 

\section{Probability of a tree} \label{sec:probtree}

We refer to branch points and leaves of the tree as nodes. The probability of generating a specific tree structure with associated divergence times and locations at each node can be written analytically since the specific diffusion path taken between nodes can be ignored. We will need the probability that a new data point does not diverge between times $s < t$ on a branch that has been followed $m$ times by previous data-points. This can straightforwardly be derived from Equation~\ref{eq:at}:
\begin{align}
P\left(\begin{array}{c} \textrm{not diverging} \\ \textrm{in } [s,t] \end{array}\right) = \exp{\left[(A(s) - A(t))\frac{\Gamma(m-\beta)}{\Gamma(m+1+\alpha)}\right]},
\label{eq:notdiverging}
\end{align}
where $A(t) = \int_0^t a(u) du$ is the cumulative rate function. 

Consider the tree of $N=4$ data points in Figure~\ref{fig:sample}. The probability of obtaining this tree structure and associated divergence times is:
\begin{align*}
&e^{-A(t_a)\frac{\Gamma(1-\beta)}{\Gamma(2+\alpha)}} \frac{a(t_a)\Gamma(1-\beta)}{\Gamma(2+\alpha)} \notag \\ 
& \times e^{-A(t_a)\frac{\Gamma(2-\beta)}{\Gamma(3+\alpha)}} \frac{1-\beta}{2+\alpha} e^{-[A(t_a)-A(t_b)]\frac{\Gamma(1-\beta)}{\Gamma(2+\alpha)}} \frac{a(t_b)\Gamma(1-\beta)}{\Gamma(2+\alpha)} \notag \\
& \times e^{-A(t_a)\frac{\Gamma(3-\beta)}{\Gamma(4+\alpha)}} \frac{\alpha+2\beta}{3+\alpha}
 \end{align*}
The first data point does not contribute to the expression. The second point contributes the first line: the first term results from not diverging between $t=0$ and $t_a$, the second from diverging at $t_a$. The third point contributes the second line: the first term comes from not diverging before time $t_a$, the second from choosing the branch leading towards the first point, the third term comes from not diverging between times $t_a$ and $t_b$, and the final term from diverging at time $t_b$. The fourth and final data point contributes the final line: the first term for not diverging before time $t_a$ and the second term for diverging at branch point $a$. 

The component resulting from the divergence and data locations for the tree in Figure~\ref{fig:sample} is
\begin{align*}
& N(x_1;0,\sigma^2) N(x_2;x_a,\sigma^2(1-t_a)) \\ \times & N(x_3;x_b,\sigma^2(1-t_b)) N(x_4;x_a,\sigma^2(1-t_a))
\end{align*}
where each data point has contributed a term. We can rewrite this as:
\begin{align}	
& N(x_a;0,\sigma^2t_a) N(x_b;x_a,\sigma^2(t_b-t_a)) \notag \\ 
 \times & N(x_1;x_b,\sigma^2(1-t_b)) \times N(x_2;x_a,\sigma^2(1-t_a)) \notag \\
 \times & N(x_3;x_b,\sigma^2(1-t_b)) N(x_4;x_a,\sigma^2(1-t_a)) \label{eq:data}
\end{align}
to see that there is a Gaussian term associated with each branch in the tree. 
\begin{figure} 
\centering
    \includegraphics[width=.8\columnwidth,clip,trim=0 0 0 0]{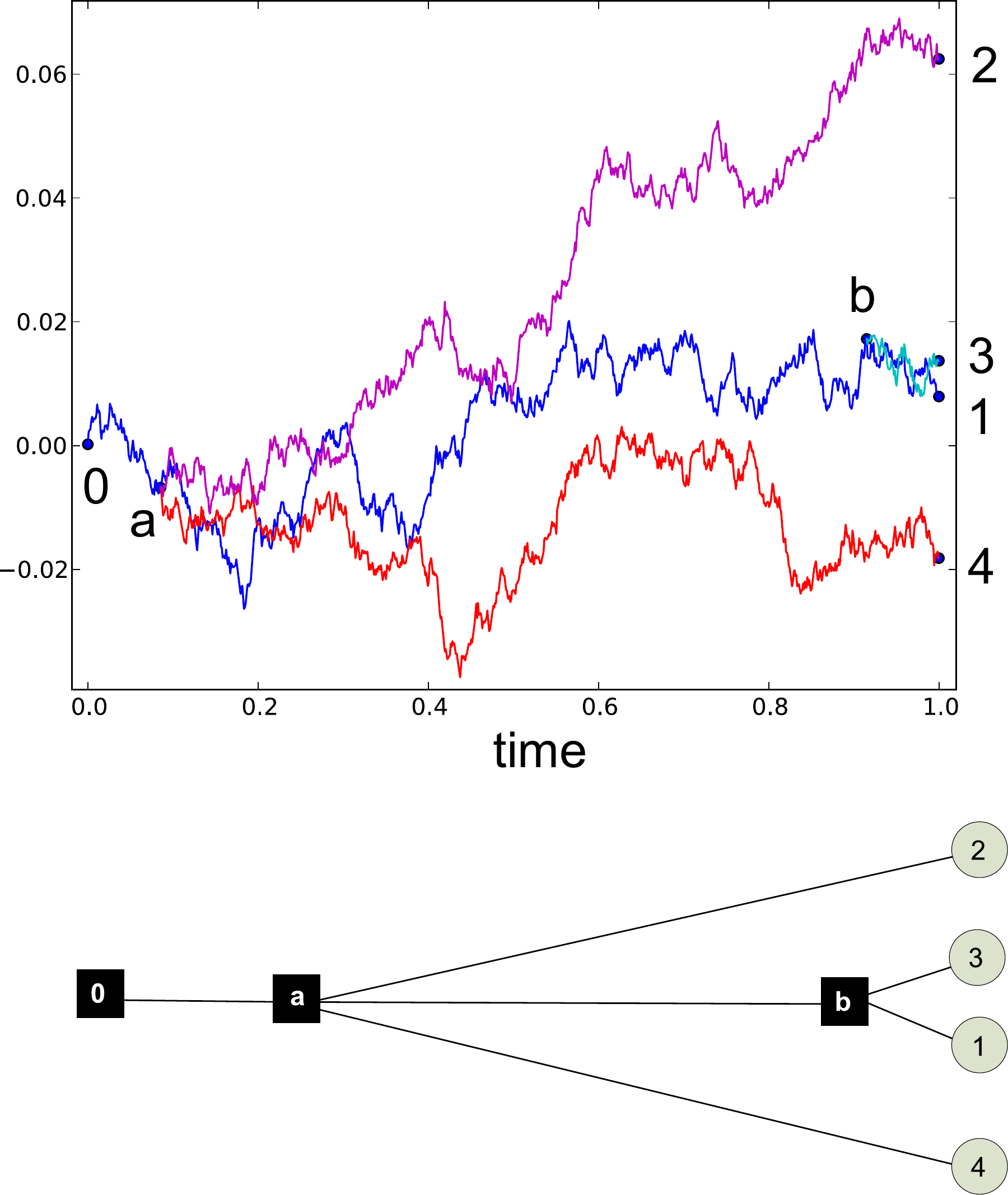}
  \vspace{-.3cm}
  \caption{A sample from the Pitman-Yor Diffusion Tree with $N=4$ datapoints and $a(t)=1/(1-t), \alpha=1, \beta=0$. Top: the location of the Brownian motion for each of the four paths. Bottom: the corresponding tree structure. Each branch point corresponds to an internal tree node.}
  \vspace{-.2cm}
  \label{fig:sample}
\end{figure}

\section{Theory} \label{sec:theory}

Now we present some important properties of the PYDT generative process. 

\newtheorem{lem}{Lemma}
\begin{lem} \label{lem}
The probability of generating a specific tree structure, divergence times, divergence locations and corresponding data set is invariant to the ordering of data points. 
\end{lem}

\begin{proof}
The probability of a draw from the PYDT can be decomposed into three components: the probability of the underlying tree structure, the probability of the divergence times given the tree structure, and the probability of the divergence locations given the divergence times. We will show that none of these components depend on the ordering of the data. Consider the tree $\mathcal{T}$ as a set of edges $\mathcal{S(T)}$ each of which we will see contributes to the joint probability density.  The tree structure $\mathcal{T}$ contains the counts of how many datapoints traversed each edge.  We denote an edge by $[ab] \in \mathcal{S(T)}$, which goes from node $a$ to node $b$ with corresponding locations $x_a$ and $x_b$ and divergence times $t_a$ and $t_b$. Let $m(b)$ be the number of samples to have passed through $b$. Denote by $\mathcal{S'(T)}=\{[ab]\in \mathcal{S(T)}:m(b)\geq 2\}$ the set of all edges traversed by $m \geq 2$ samples (for divergence functions which ensure divergence before time $1$ this is the set of all edges not connecting to leaf nodes).

\emph{Probability of the tree structure.}  For segment $[ab]$, let $i$ be the index of the sample which diverged to create the branch point at $b$, thereby contributing a factor 
\begin{align}
\frac{a(t_b)\Gamma(i-1-\beta)}{\Gamma(i+\alpha)}. \label{eq:at2}
\end{align}
Let the number of branches from $b$ be $K_b$, and the number of samples which followed each branch be $\{n_k^b : k \in [1 \dots K_b]\}$. The total number of datapoints which  traversed edge $[ab]$ is $m(b)=\sum_{j=1}^{K_b} n_k^b$. It can be shown (see Appendix~\ref{app:probstruct}) that the factor associated with this branching structure for the data points after $i$ is
\begin{align*}
\frac{\prod_{k=3}^{K_b} [ \alpha + (k-1) \beta] \Gamma(i+\alpha) \prod_{l=1}^{K_b} \Gamma(n_l^b-\beta)}
{ \Gamma(i-1+\beta) \Gamma(m(b)+\alpha) }
\end{align*}
Multiplying by the contribution from data point $i$ in Equation~\ref{eq:at2} we have
\begin{align}
 \frac{a(t_b) \prod_{k=3}^{K_b} [ \alpha + (k-1) \beta] \prod_{l=1}^{K_b} \Gamma(n_l^b-\beta)}
{  \Gamma(m(b)+\alpha) } \label{eq:struct}
\end{align}
Each segment $[ab] \in \mathcal{S'(T)}$ contributes such a term. Since this expression does not depend on the ordering of the branching events, the overall factor does not either. 

\emph{Probability of divergence times.} The $m(b)-1$ points that followed the first point along this path did not diverge before time $t_b$ (otherwise $[ab]$ would not be an edge), which from Equation~\ref{eq:notdiverging} we see contributes a factor
\begin{align}
&\prod_{i = 1}^{m(b)-1} \exp\left[(A(t_a) - A(t_b)) \frac{\Gamma(i-\beta)}{\Gamma(i+1+\alpha)}\right] \notag \\
&= \exp{\left[(A(t_a) - A(t_b)) H^{\alpha,\beta}_{m(b)-1}\right]} \label{eq:nd}
\end{align}
where we define $H^{\alpha,\beta}_n=\sum_{i=1}^n \frac{\Gamma(i-\beta)}{\Gamma(i+1+\alpha)}$.  All edges $[ab] \in \mathcal{S'(T)}$ contribute the expression in Equation~\ref{eq:nd}, resulting in a total contribution
\begin{align}
 \prod_{[ab] \in \mathcal{S'(T)}}  \exp{\left[(A(t_a) - A(t_b)) H^{\alpha,\beta}_{m(b)-1}\right]}
 \label{eq:times}
\end{align}
This expression does not depend on the ordering of the datapoints. 

\emph{Probability of node locations.} Generalizing Equation~\ref{eq:data} it is clear that each edge contributes a Gaussian factor, resulting an overall factor:
\begin{align}
 \prod_{[ab] \in \mathcal{S(T)}} \textrm{N}(x_b; x_a, \sigma^2 (t_b - t_a) I) \label{eq:locs}
\end{align}

The overall probability of a specific tree, divergence times and node locations is given by the product of Equations~\ref{eq:struct},~\ref{eq:times} and~\ref{eq:locs}, none of which depend on the ordering of the data. 
\end{proof}

The term $\prod_{k=3}^{K_b} [ \alpha + (k-1) \beta]$ in Equation~\ref{eq:struct} can be calculated efficiently depending on the value of $\beta$. For $\beta=0$ we have $\prod_{k=3}^{K_b} \alpha = \alpha ^ {K-2}$. 
For $\beta \neq 0$ we have
\begin{align*}
\prod_{k=3}^{K_b} [ \alpha + (k-1) \beta] &= \beta^{K_b-2} \prod_{k=3}^{K_b} [ \alpha/\beta + (k-1) ] \\
&= \frac{\beta^{K_b-2}\Gamma(\alpha/\beta + K_b)}{\Gamma(\alpha/\beta+2)}
\end{align*}

\newtheorem{thm}{Theorem}
\begin{thm}
The Pitman-Yor Diffusion Tree defines an infinitely exchangeable distribution over data points. 
\end{thm}

\begin{proof}
Summing over all possible tree structures, and integrating over all branch point times and locations, by Lemma~\ref{lem} we have infinite exchangeability. 
\end{proof}

\newtheorem{cor}{Corollary}
\begin{cor}
There exists a prior $\nu$ on probability measures on $\mathbb{R}^D$ such that the samples $x_1,x_2,\dots$ generated by a PYDT are conditionally independent and identically distributed (iid) according to $\mathcal{F}\sim\nu$, that is, we can represent the PYDT as
$$PYDT(x_1,x_2,\dots)=\int \left( \prod_i \mathcal{F}(x_i) \right) d\nu(\mathcal{F})$$. 
\end{cor}

\begin{proof}
Since the PYDT defines an infinitely exchangeable process on data points, the result follows directly by de Finetti's Theorem~\citep{definetti}. 
\end{proof}

Another way of expressing Corollary 1 is that data points $x_1,\dots,x_N$ sampled from the PYDT could equivalently have been sampled by first sampling a probability measure $\mathcal{F}\sim \nu$, then sampling $x_i \sim \mathcal{F}$ iid for all $i$ in $\{1,\dots,N\}$. For divergence functions such that $A(1)$ is infinite, the probability measure $\mathcal{F}$ is continuous almost surely. 

\begin{lem}
The PYDT reduces to the Diffusion Dirichlet Tree~\citep{Neal2001b} in the case $\alpha=\beta=0$.  
\end{lem}

\begin{proof}
This is clear from the generative process: for $\alpha=\beta=0$ there is zero probability of branching at a previous branch point (assuming continuous cumulative divergence function $A(t)$). The probability of diverging in the time interval $[t,t+dt]$ from a branch previously traversed by $m$ datapoints becomes:
\begin{align}
\frac{a(t)\Gamma(m-0)dt}{\Gamma(m+1+0)} = \frac{a(t)(m-1)!dt}{m!} = \frac{a(t)dt}{m} 
\end{align}
as for the DDT.
\end{proof}

It is straightforward to confirm that the DDT probability factors are recovered when $\alpha=\beta=0$. In this case $K=2$ since non-binary branch points have zero probability, so Equation~\ref{eq:struct} reduces as follows:
\begin{align*}
 \frac{a(t_b) \prod_{l=1}^{K=2} \Gamma(n_l^b-0)}{  \Gamma(m(b)+0) } =\frac{a(t_b) (b_1-1)!(b_2-1)!}{(m(b)-1)!}
\end{align*}
as for the DDT. Equation~\ref{eq:times} also reduces to the DDT expression since 
\begin{align*}
H^{0,0}_n=\sum_{i=1}^n \frac{\Gamma(i-0)}{\Gamma(i+1+0)} = \sum_{i=1}^n \frac{(i-1)!}{i!} = \sum_{i=1}^n \frac1{i} = H_n
\end{align*}
where $H_n$ is the $n$-th Harmonic number. 

For the purpose of this paper we use the divergence function $a(t) = \frac{c}{1-t}$, with ``smoothness'' parameter $c>0$. Larger values $c>1$ give smoother densities because divergences typically occur earlier, resulting in less dependence between the datapoints. Smaller values $c<1$ give rougher more ``clumpy'' densities with more local structure since divergence typically occurs later, closer to $t=1$. For this divergence function we have $A(t)=-c\log{(1-t)}$. 

Equation~\ref{eq:times} factorizes into a term for $t_a$ and $t_b$. Collecting such terms from the branches attached to an internal node $b$ the factor for $t_b$ for the divergence function $a(t)=c/(1-t)$ is
\begin{align}
P(t_b|\mathcal{T}) &= a(t_b)\exp{ \left[ A(t_b) \left( \sum_{k=1}^{K_b} H^{\alpha,\beta}_{n_k^b-1} - H^{\alpha,\beta}_{m(b)-1} \right) \right] } \notag \\
&= c (1-t_b) ^ { c J^{\alpha,\beta}_{\mathbf{n}^b} - 1 } \label{eq:tprior}
\end{align}
where $J^{\alpha,\beta}_{n^b}=H^{\alpha,\beta}_{\sum_{k=1}^{K} n^b_k-1} - \sum_{k=1}^{K} H^{\alpha,\beta}_{n^b_k-1}$ with $\mathbf{n}^b \in \mathbb{N}^K$. 

This generalization of the DDT allows non-binary tree structures to be learnt. By varying $\alpha$ we can move between flat (large $\alpha$) and hierarchical clusterings (small $\alpha$), as shown in Figure~\ref{fig:varyalpha}. 

\begin{figure} 
\centering
    \includegraphics[width=.8\columnwidth,clip,trim=0 0 0 0]{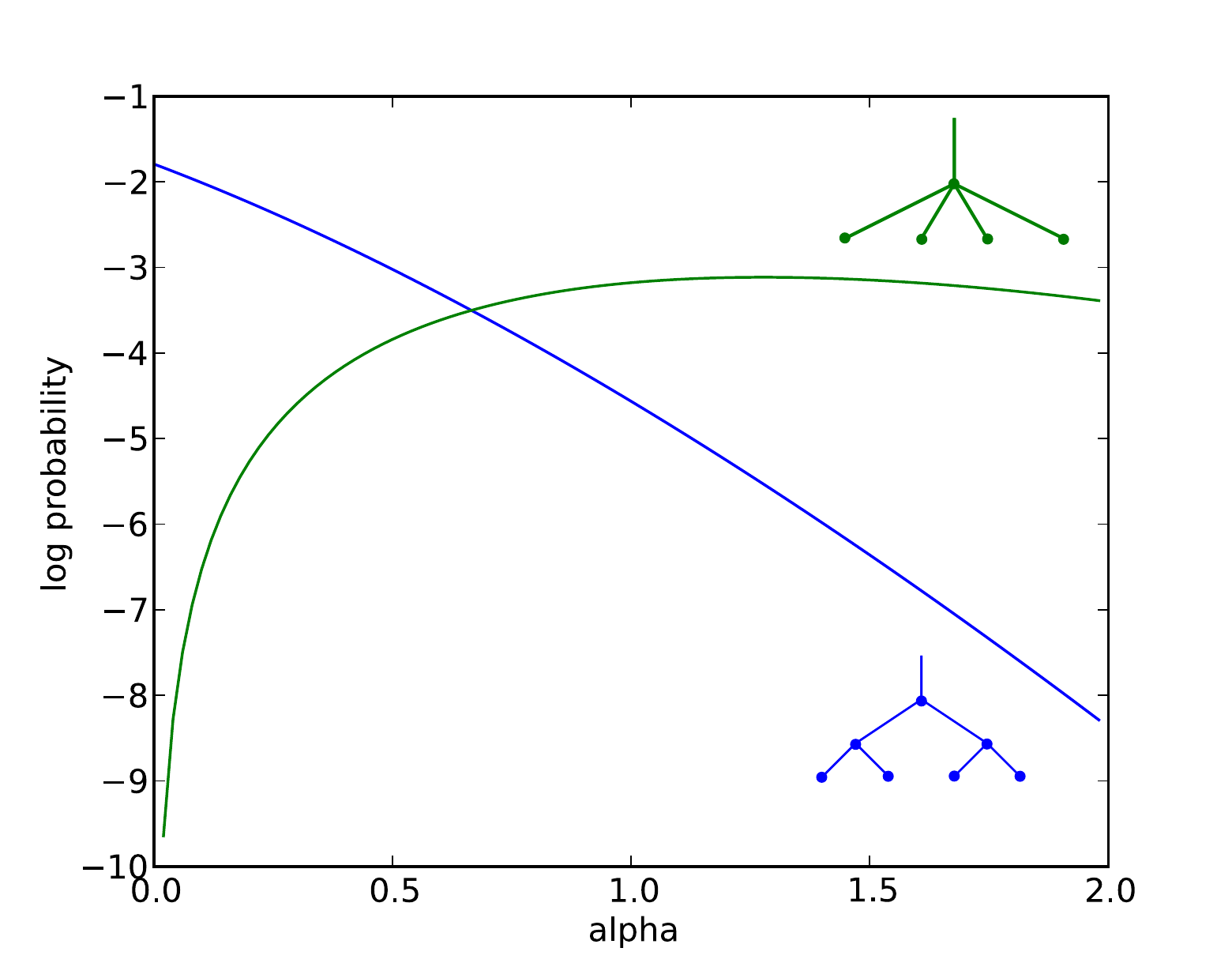}
  \vspace{-.3cm}
  \caption{The effect of varying $\alpha$ on the log probability of two tree structures, indicating the types of tree preferred. Small $\alpha < 1$ favors binary trees while larger values of $\alpha$ favors higher order branching points.}
  \vspace{-.2cm}
  \label{fig:varyalpha}
\end{figure}

\begin{figure} 
\centering
    \includegraphics[width=.8\columnwidth,clip,trim=0 0 0 0]{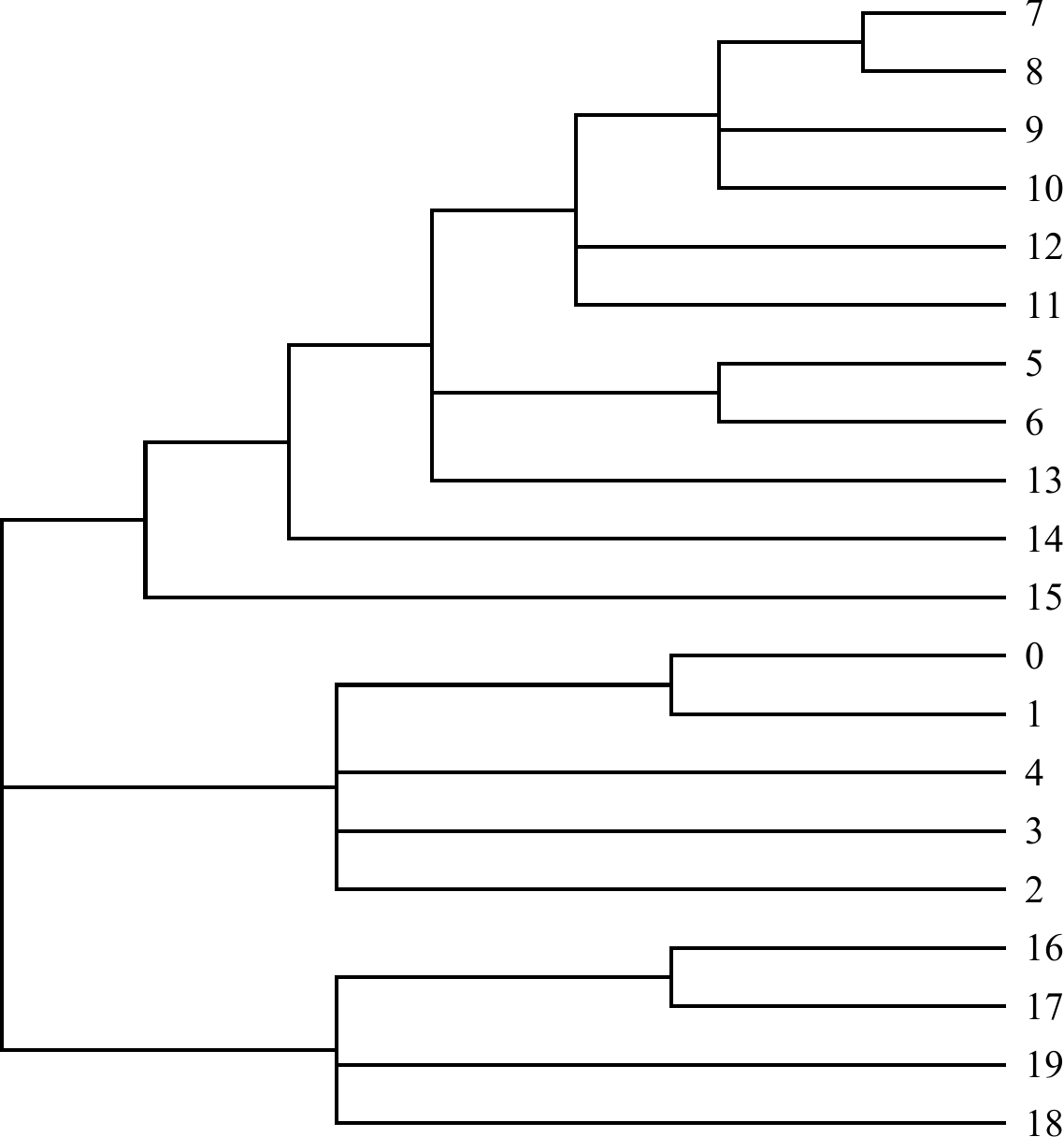}
  \vspace{-.3cm}
  \caption{A sample from the Pitman-Yor Diffusion Tree with $N=20$ datapoints and $a(t)=1/(1-t), \alpha=1, \beta=0$ showing the branching structure including non-binary branch points.}
  \vspace{-.2cm}
  \label{fig:treeN20c1a1b0}
\end{figure}

\begin{figure} 
\centering
    \subfigure[$c=1,\alpha=0,\beta=0$ (DDT)]
    {\includegraphics[width=.45\columnwidth]{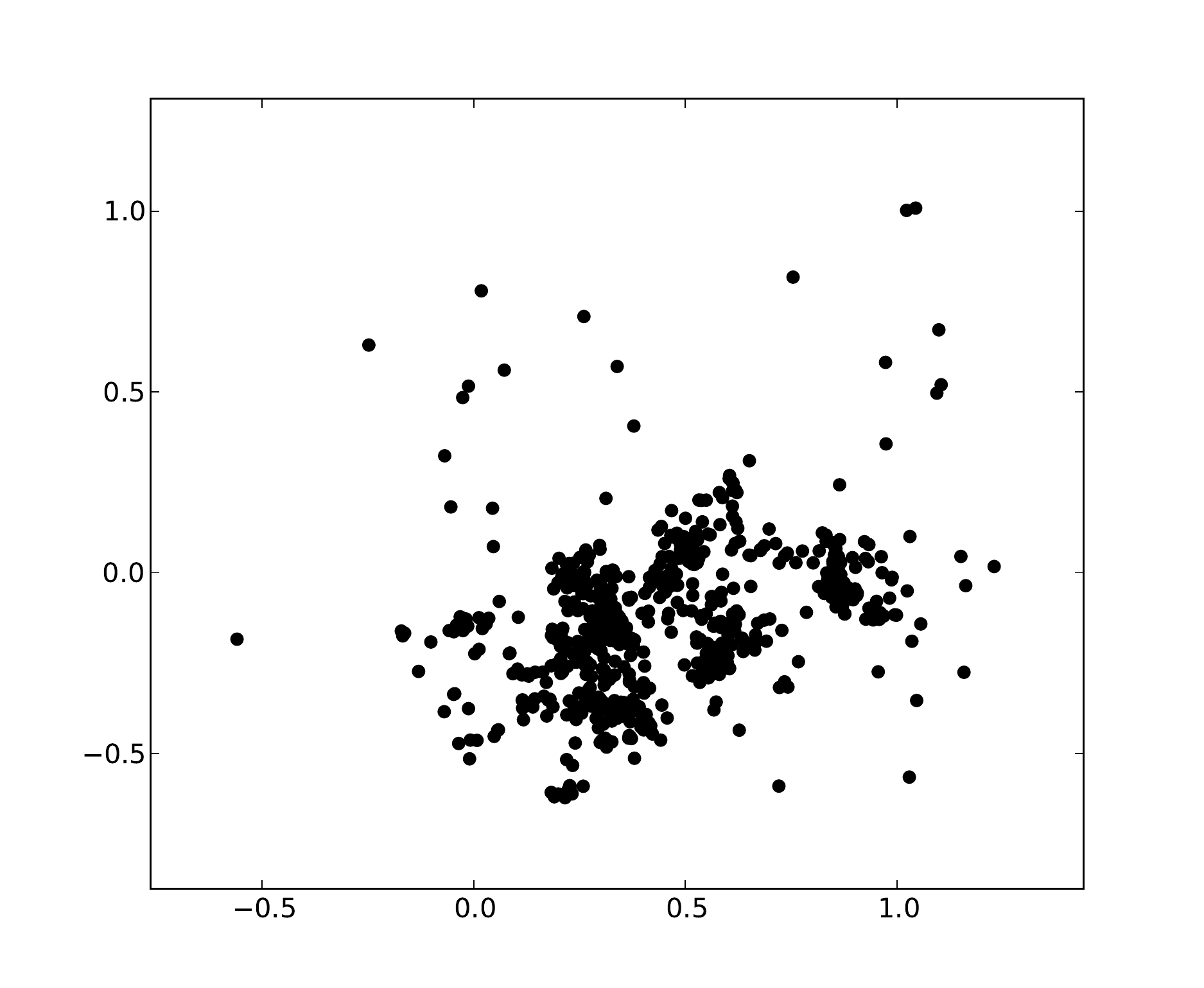}}
    \subfigure[$c=1,\alpha=0.5,\beta=0$]
    {\includegraphics[width=.45\columnwidth]{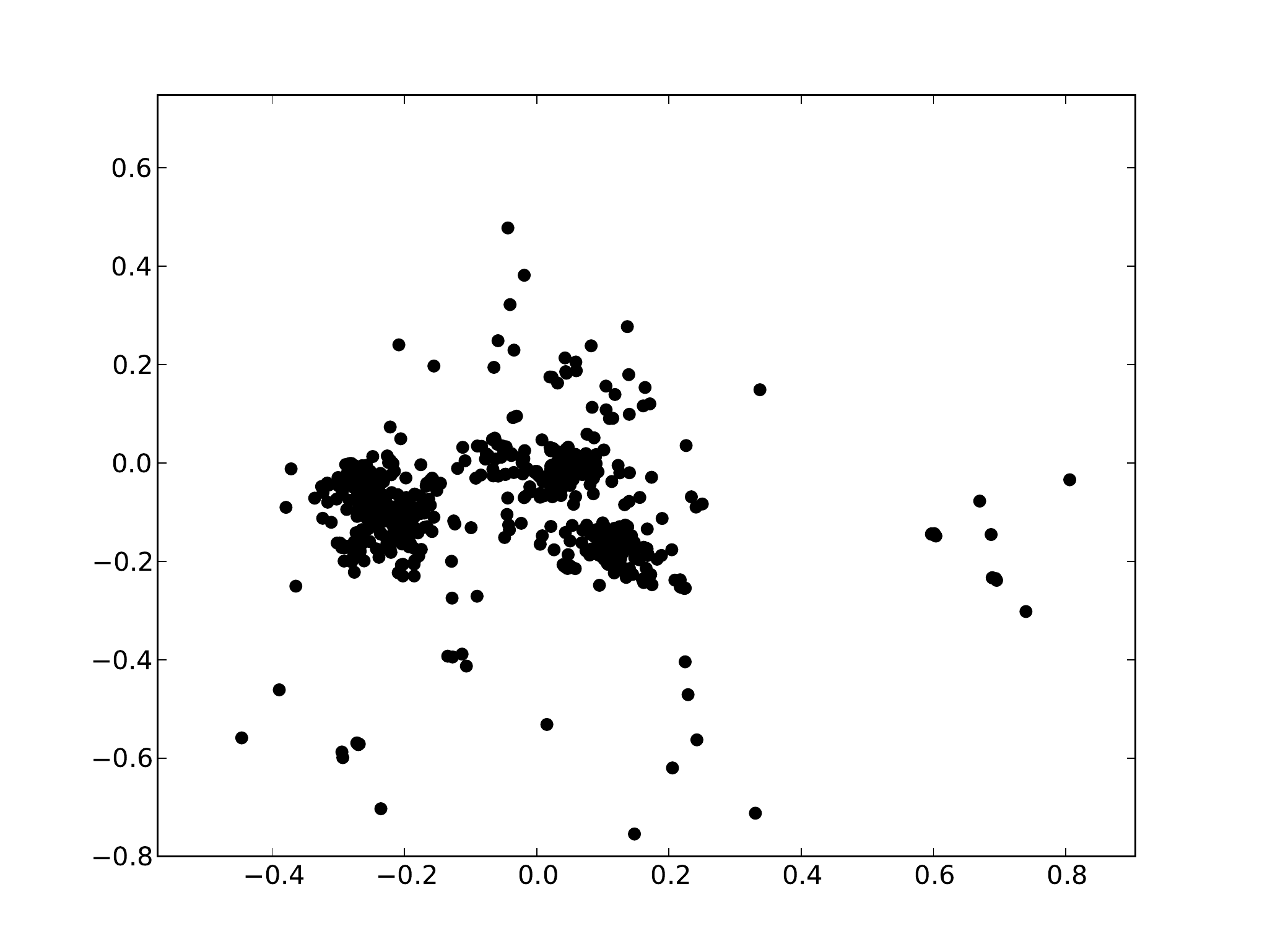}}
        \subfigure[$c=1,\alpha=1,\beta=0$]
    {\includegraphics[width=.45\columnwidth]{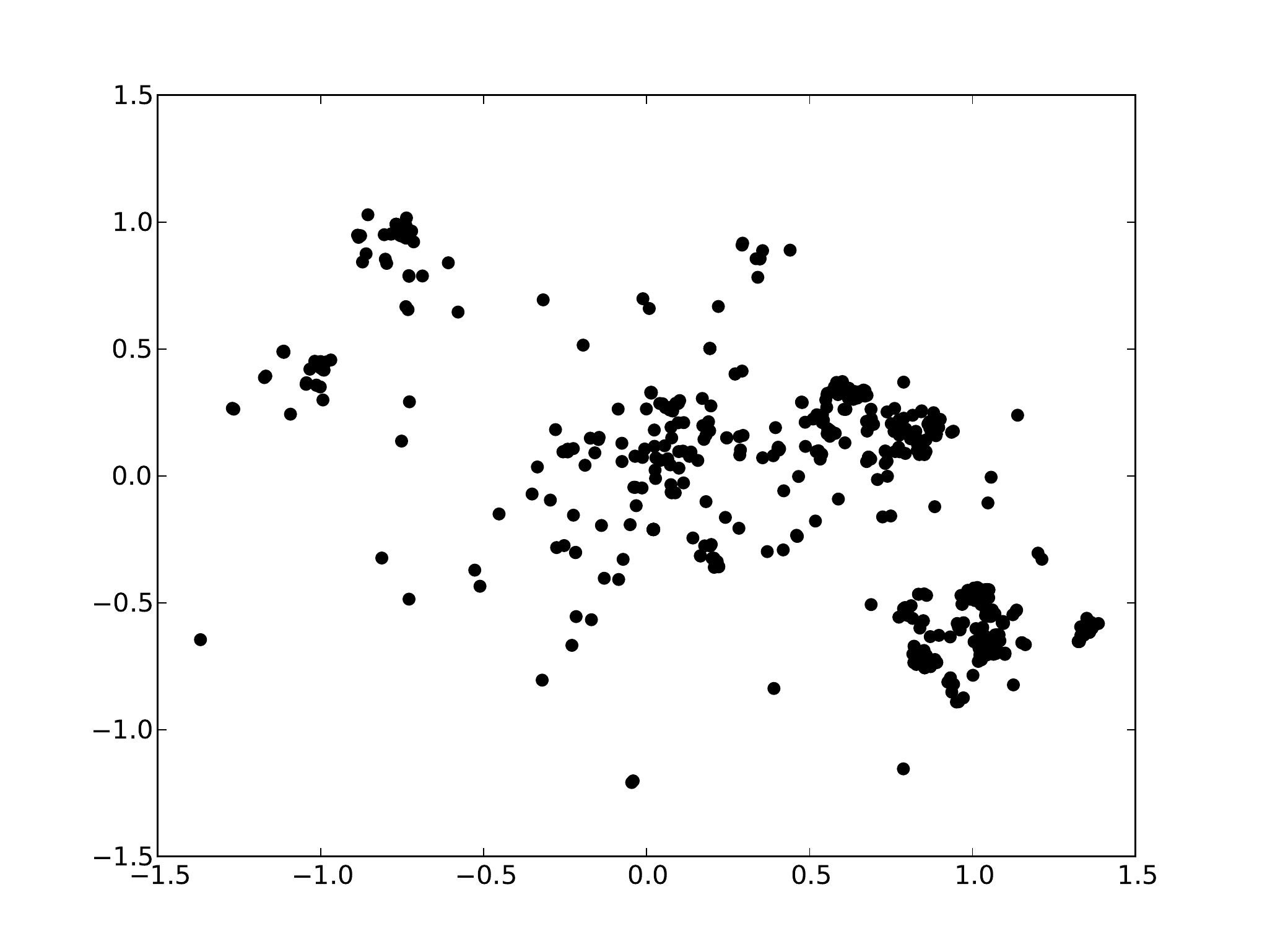}}
        \subfigure[$c=3,\alpha=1.5,\beta=0$]
    {\includegraphics[width=.45\columnwidth]{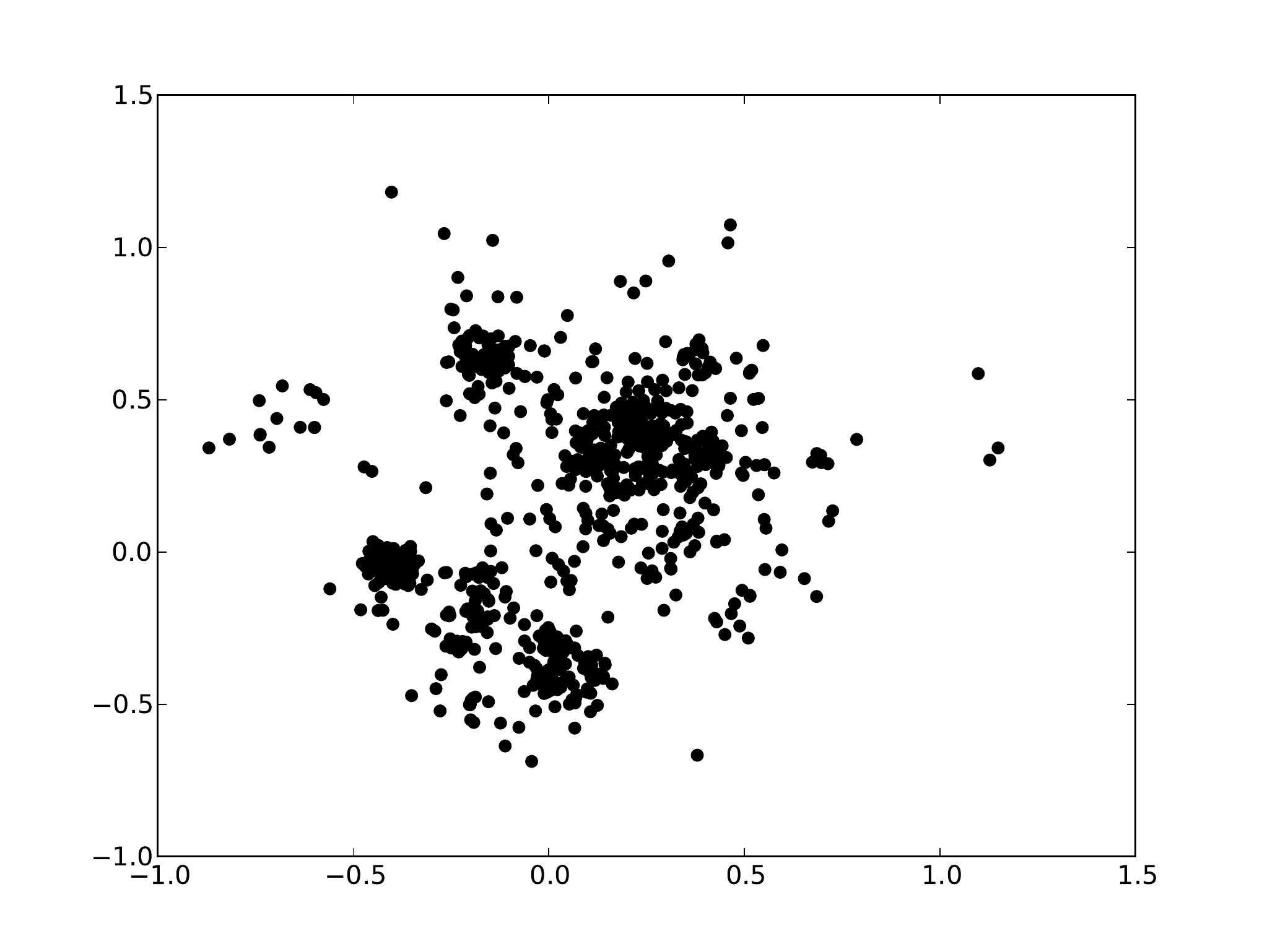}}
  \caption{Samples from the Pitman-Yor Diffusion Tree with $N=1000$ datapoints in $D=2$ dimensions and $a(t)=c/(1-t)$. As $\alpha$ increases more obvious clusters appear.}
  \label{fig:samples}
\end{figure}

\section{Model} \label{sec:model}

To complete the model we must specify a likelihood function for the data given the leaf locations of the PYDT, and priors on the hyperparameters. We use a Gaussian observation model for multivariate continuous data and a probit model for binary vectors. We specify the following priors on the hyperparameters:
\begin{align*}
\alpha \sim \text{G}(a_\alpha,b_\alpha) &\qquad
\beta \sim \text{Beta}(a_\beta,b_\beta) \\
c \sim \text{G}(a_c,b_c) &\qquad
1/\sigma^2 \sim \text{G}(a_{\sigma^2}, b_{\sigma^2})
\end{align*}
where $G(a,b)$ is a Gamma distribution with shape $a$ and rate $b$. In all experiments we used $a_\alpha=2,b_\alpha=.5,a_\beta=1,b_\beta=1,a_c=1,b_c=1,a_{\sigma^2}=1, b_{\sigma^2}=1$.

\section{Inference} \label{sec:inference}

We propose two inference algorithms: an MCMC sampler and a more computationally efficient greedy EM algorithm. Both algorithms marginalize out the locations of internal nodes using belief propagation, and are capable of learning the hyperparameters $c,\sigma^2,\alpha$ and $\beta$ if desired. 

\subsection{MCMC sampler} 

We construct an MCMC sampler to explore the posterior over the tree structure, divergence times and hyperparameters. To sample the structure and divergence times a subtree is chosen uniformly at random to be detached (the subtree may be a single leaf node). To propose a new position in the tree for the subtree, we follow the procedure for generating a new sample on the remaining tree. The subtree is attached wherever divergence occurred, which may be on a segment, in which case a new parent node is created, or at an existing internal node, in which case the subtree becomes a child of that node. If divergence occurred at a time later than the divergence time of the root of the subtree we must repeat the procedure until this is not the case. The marginal likelihood of the new tree is calculated, marginalizing over the internal node locations, and excluding the structure and divergence time contribution since this is accounted for by having sampled the new location according to the prior. The ratio to the marginal likelihood for the original tree gives the Metropolis factor used to determine whether this move is accepted. Unfortunately it is not possible to slice sample the position of the subtree as in~\cite{Neal2003a} because of the atoms in the prior at each branch point. 

\paragraph{Smoothness hyperparameter $c$.} From Equation~\ref{eq:tprior} the Gibbs conditional for $c$ is
\begin{align}
\text{G}\left( a_c + |\mathcal{I}| 
, b_c + \sum_{i \in \mathcal{I}} J^{\alpha,\beta}_{n^i} \log{(1-t_i)} \right) \label{eq:updatec}
\end{align}
where $\mathcal{I}$ is the set of internal nodes of the tree. 

\paragraph{Data variance $\sigma^2$.} It is straightforward to sample $1/\sigma^2$ given divergence locations. Having performed belief propagation it is easy to jointly sample the divergence locations using a pass of backwards sampling. From Equation~\ref{eq:locs} the Gibbs conditional for the precision $1/\sigma^2$ is then
\begin{align}
\text{G}( a_{\sigma^2}, b_{\sigma^2} )
\prod_{[ab] \in \mathcal{S(T)}} \text{G}\left( D/2 + 1
, \frac{|| x_a-x_b ||^2}{2(t_b-t_a)}\right) \label{eq:updates2}
\end{align}
where $|| \cdot ||$ denotes Euclidean distance. 

\paragraph{Pitman-Yor hyperparameters $\alpha,\beta$.} We use slice sampling~\citep{Neal2003} to sample $\alpha$ and $\beta$. We reparameterize in terms of the logarithm of $\alpha$ and the logit of $\beta$ to extend the domain to the whole real line. The terms required to calculate the conditional probability are those in Equations~\ref{eq:struct} and~\ref{eq:times}. 

\subsection{Greedy Bayesian EM algorithm} 

As an alternative to MCMC here we use a Bayesian EM algorithm to approximate the marginal likelihood for a given tree structure, which is then used to drive a greedy search over tree structures, following our work in~\cite{knowles2011icml}. 

\paragraph{EM algorithm.} In the E-step, we use message passing to integrate over the locations and hyperparameters. In the M-step we maximize the lower bound on the marginal likelihood with respect to the divergence times. 
For each node $i$ with divergence time $t_i$ we have the constraints $t_p < t_i < \min{(t_l,t_r)}$ where $t_l,t_r,t_p$ are the divergence times of the left child, right child and parent of $i$ respectively. 

We jointly optimize the divergence times using LBFGS~\citep{lbfgs}. Since the divergence times must lie within $[0,1]$ we use the reparameterization $s_i=\log{[t_i/(1-t_i)]}$ to extend the domain to the real line, which we find improves empirical performance. From Equations~\ref{eq:locs} and~\ref{eq:tprior} the lower bound on the log evidence is a sum over all branches $[pi]$ of expressions of the form:
\begin{align}
 &(\langle c \rangle  J^{\alpha,\beta}_{\mathbf{n}^i} -1 ) \log{(1-t_i)} - \frac{D}2 \log{(t_i-t_p)} - \langle \frac1{\sigma^2} \rangle \frac{b_{[pi]}}{t_i-t_p} \label{eq:em}
 \end{align}
 where $b_{[pi]} = \frac12 \sum_{d=1}^D \mathbb{E}[(x_{di}-x_{dp} )^2]$, $x_{di}$ is the location of node $i$ in dimension $d$, and $p$ is the parent of node $i$. The full lower bound is the sum of such terms over all nodes. The expectation required for $b_{[pi]}$ is readily calculated from the marginals of the locations after message passing. Differentiating to obtain the gradient with respect to $t_i$ is straightforward so we omit the details. Although this is a constrained optimization problem (branch lengths cannot be negative) it is not necessary to use the log barrier method because the $1/(t_i-t_p)$ terms in the objective implicitly enforce the constraints. 

\paragraph{Hyperparameters.} We use variational inference to learn Gamma posteriors on the inverse data variance $1/\sigma^2$ and smoothness $c$. The variational updates for $c$ and $1/\sigma^2$ are the same as the conditional Gibbs distributions in Equations~\ref{eq:updatec} and~\ref{eq:updates2} respectively. We optimize $\alpha$ and $\beta$ by coordinate descent using golden section search on the terms in Equations~\ref{eq:struct} and~\ref{eq:times}.

\paragraph{Search over tree structures} The EM algorithm approximates the marginal likelihood for a fixed tree structure $\mathcal{T}$.  We maintain a list of $K$-best trees (typically $K=10$) which we find gives good empirical performance. Similarly to the sampler, we search the space of tree structures by detaching and re-attaching subtrees. We choose which subtree to detach at random. We can significantly improve on re-attaching at random by calculating the local contribution to the evidence that would be made by attaching the root of the subtree to the midpoint of each possible branch and at each possible branch point. We then run EM on just the three best resulting trees. We found construction of the initial tree by sequential attachment of the data points using this method to give very good initializations. 

\subsection{Predictive distribution} \label{sec:pred}

To calculate the predictive distribution for a specific tree we compute the distribution for a new data point conditioned on the posterior location and divergence time marginals. Firstly, we calculate the probability of diverging from each branch according to the data generating process described in Section~\ref{sec:genprocess}. Secondly we draw several (typically three) samples of when divergence from each branch occurs. Finally we calculate the Gaussian at the leaves resulting from Brownian motion starting at the sampled divergence time and location up to to $t=1$. This results in a predictive distribution represented as a weighted mixture of Gaussians. Finally we average the density from a number of samples from the sampler or the $K$-best trees found by the EM search algorithm. 

\subsection{Likelihood models}

Connecting our PYDT module to different likelihood models is straightforward: we use a Gaussian observation model and a probit model for binary vectors. The MCMC algorithm slice samples auxiliary variables and the EM algorithm uses EP~\citep{Minka2001a} on the probit factor, implemented using the runtime component of the Infer.NET framework~\cite{InferNET10}. 

\section{Results} \label{sec:results}

We present results on synthetic and real world data, both continuous and binary. 

\subsection{Synthetic data}

We first compare the PYDT to the DDT on a simple synthetic dataset with $D=2,N=100$, sampled from the density 
$$ f(x,y)= \frac14 \sum_{\bar{x} \in [-1,1]} \sum_{\bar{y} \in [-1,1]} N(x;\bar{x},1/8)N(y;\bar{y},1/8)$$
The optimal trees learnt by 100 iterations of the greedy EM algorithm are shown in Figure~\ref{fig:synthetic}. While the DDT is forced to arbitrarily choose a binary branching structure over the four equi-distant clusters, the PYDT is able to represent the more parsimonious solution that the four clusters are equally dependent. Both models find the fine detail of the individual cluster samples which may be undesirable; investigating whether learning a noise model for the observations alleviates this problem is a subject of future work. 

\begin{figure} 
\centering
    \subfigure[DDT]
    {\includegraphics[width=.49\columnwidth]{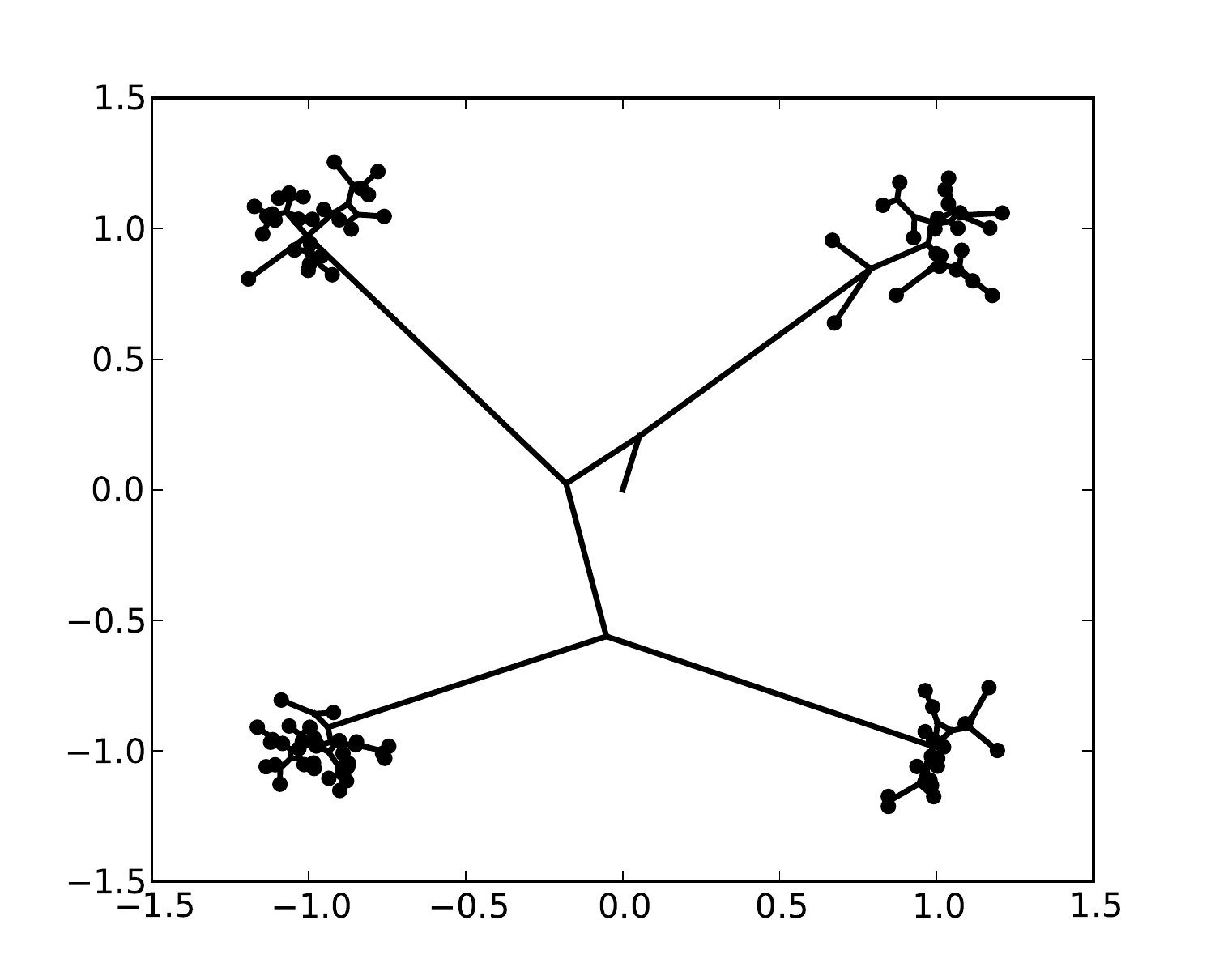}}
    \subfigure[PYDT]
    {\includegraphics[width=.49\columnwidth]{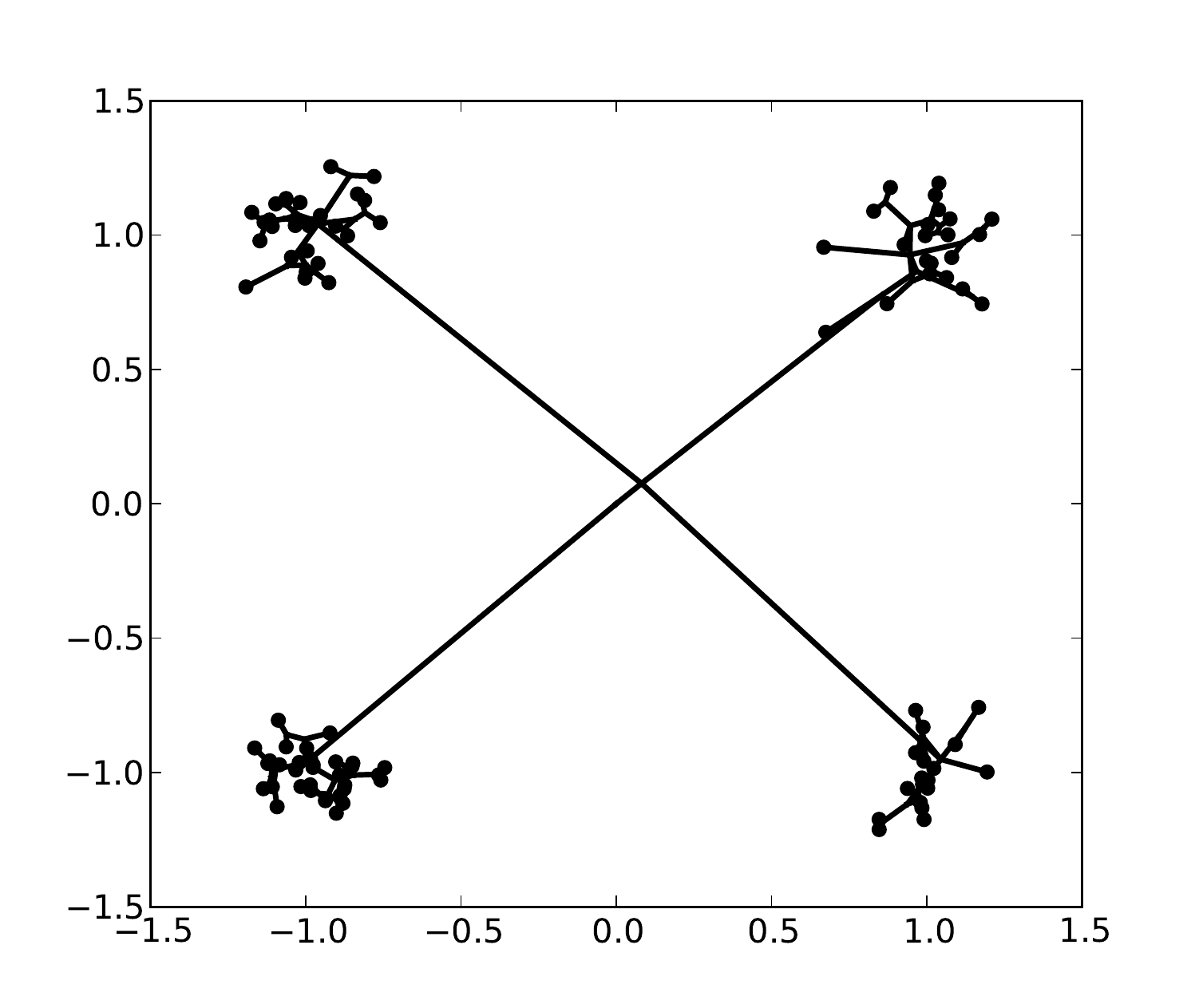}}
  \caption{Optimal trees learnt by the greedy EM algorithm for the DDT and PYDT on a synethic dataset with $D=2, N=100$.}
  \label{fig:synthetic}
\end{figure}

\subsection{Density modeling}

In~\cite{MackayGDPS} the DDT was shown to be an excellent density model on a $D=10,N=228$ dataset of macaque skull measurements, outperforming a kernel density and infinite mixture of Gaussians, and sometimes the Gaussian process density sampler itself. We compare the PYDT to the DDT on the same dataset, using the same data preprocessing and same three train test splits ($N_\text{train}=200,N_\text{test}=28$) as~\cite{MackayGDPS}. The performance using the MCMC sampler is shown in Figure~\ref{fig:macaques}. The PYDT finds trees with higher marginal likelihood than the DDT, which corresponds to a moderate improvement in predictive performance. Inference in the PYDT is actually slightly more efficient computationally than in the DDT because the on average smaller number of internal nodes reduces the cost of belief propagation over the divergence locations, which is the bottleneck of the algorithm. 

\begin{figure} 
\centering
    \includegraphics[width=.8\columnwidth]{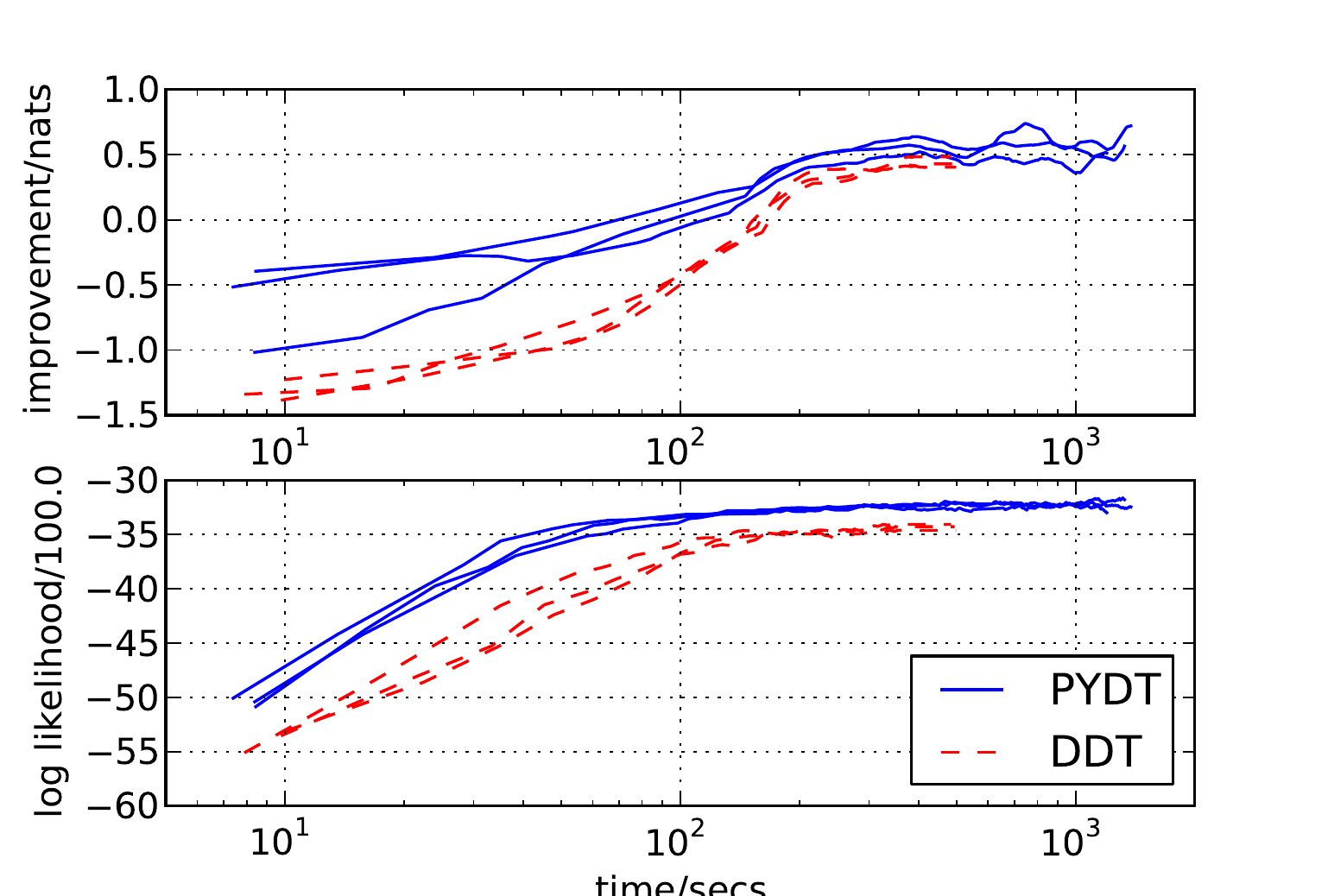}
  \vspace{-.2cm}
  \caption{Density modeling of the $D=10,N=200$ macaque skull measurement dataset of~\cite{MackayGDPS}. \emph{Top}: Improvement in test predictive likelihood compared to a kernel density estimate. \emph{Bottom}: Marginal likelihood of current tree. The shared x-axis is computation time in seconds.}
  \vspace{-.2cm}
  \label{fig:macaques}
\end{figure}

\subsection{Binary example}

To demonstrate the use of an alternative observation model we use a probit observation model in each dimension to model 102-dimensional binary feature vectors relating to attributes (e.g. being warm-blooded, having two legs) of 33 animal species from~\cite{structuralformtenenbaum}. The MAP tree structure learnt using EM, is shown in Figure~\ref{fig:animals}, is intuitive, with subtrees corresponding to land mammals, aquatic mammals, reptiles, birds, and insects (shown by colour coding). Note that penguins cluster with aquatic species rather than birds, which is not surprising since there are attributres such as ``swims'', ``flies'' and ``lives in water''. 

\begin{figure} 
\centering
    \includegraphics[width=.65\columnwidth,clip,trim=0 0 0 0]{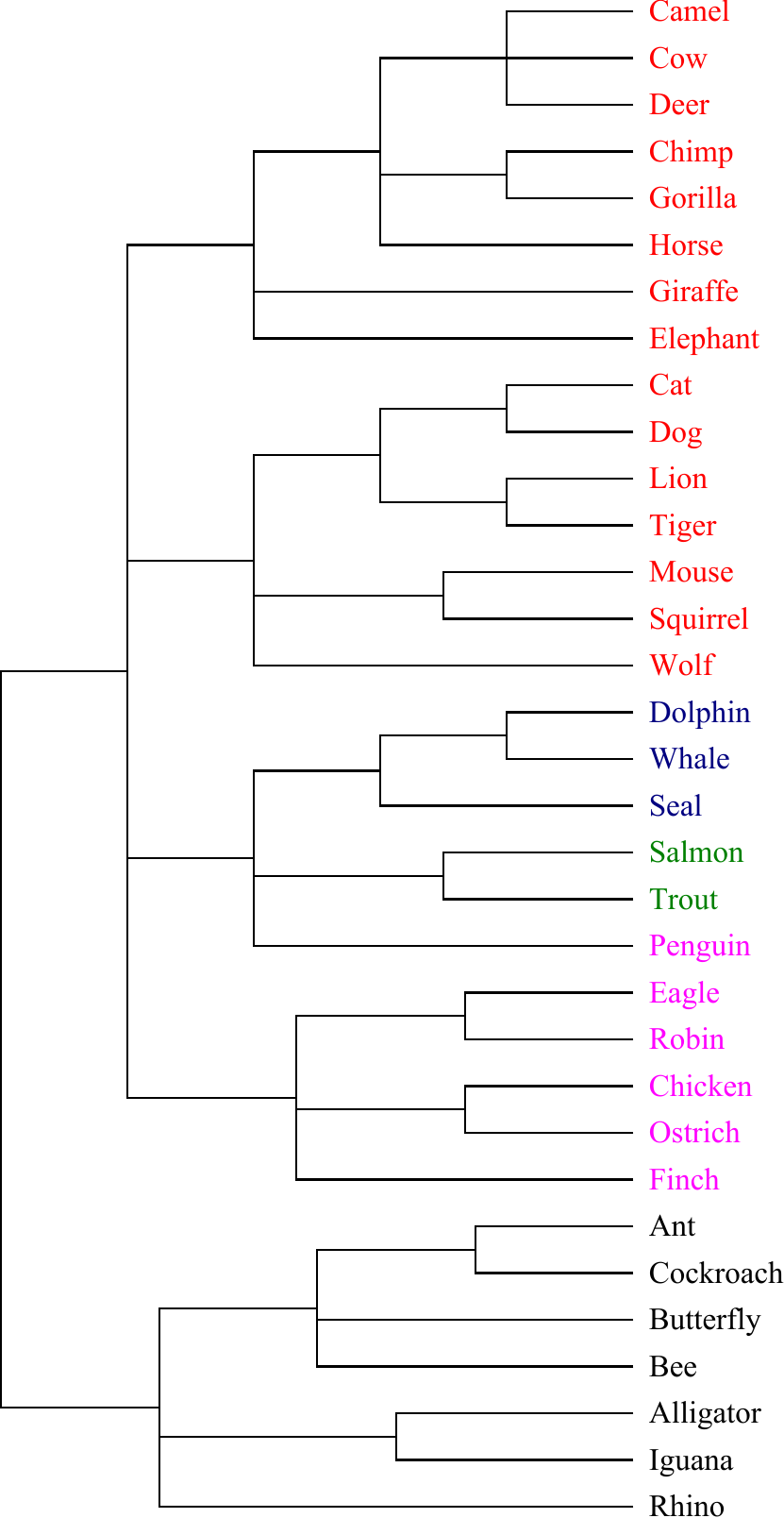}
  \vspace{-.3cm}
  \caption{Tree structure learnt for the animals dataset of~\cite{structuralformtenenbaum}.}
  \vspace{-.4cm}
  \label{fig:animals}
\end{figure}

\section{Conclusion}

We have introduced the Pitman-Yor Diffusion Tree, a Bayesian nonparametric prior over tree structures with arbitrary branching structure at each branch point. We have shown the PYDT defines an infinitely exchangeable distribution over data points. We demonstrated an MCMC sampler and Bayesian EM with greedy search, both using message passing on the tree structure. More advanced MCMC methods could be of use here. Quantitatively we have shown a modest improvement relative to the DDT on a density estimation task. However, we see improved interpretability as the key benefit of removing the restriction to binary trees, especially since hierarchical clustering is typically used as a data exploration tool. Qualitatively, we have shown the PYDT can find simpler, more interpretable representations of data than the DDT. To encourage the use of the PYDT by the community we will make our code publicly available. 

In ongoing work we use the PYDT to learn hierarchical structure over latent variables in models including Hidden Markov Models, specifically in part of speech tagging~\citep{Kupiec1992225} where a hierarchy over the latent states aids interpretability, and Latent Dirichlet Allocation, where it is intuitive that topics might be hierarchically clustered~\citep{Blei2004}. Another interesting direction would be to use the prior over branching structures implied by the PYDT in the annotated hierarchies model of~\cite{Roy2007}. 

\appendix

\section{Probability of tree structure} \label{app:probstruct}

For segment $[ab]$, recall that $i$ is the index of the sample which created the branch point at $b$. Thus $i-1$ samples did not diverge at $b$ so do not contribute any terms. Let the final number of branches from $b$ be $K_b$, and the number of samples which followed each branch be $\mathbf{n}^k:=\{n_k^b : k \in [1 \dots K_b]\}$. The probability of the $i$-th sample having diverged to form the branch point is $\frac{a(t_b)\Gamma(i-1-\beta)}{\Gamma(i+\alpha)}$. Now we wish to calculate the probability of thefinal branching structure at $b$. Following the divergence of sample $i$ there are $K_b-2$ samples who form new branches from the same point, contributing $\alpha+(k-1)\beta$ to the numerator for $k \in \{3,\dots,K_b\}$. Let $c_l$ be the number of samples having previously followed path $l$, so that $c_l$ ranges from $1$ to $n_l^b-1$ (apart from $c_1$ which only ranges from $i-1$ to $n_1^b-1$). The $j$-th sample contributes a factor $j-1+\alpha$ to the denominator. The total number of datapoints which  traversed edge $[ab]$ is $m(b)=\sum_{j=1}^{K_b} n_k^b$. The factor associated with this branch point is then: 
\begin{align*}
& \frac{\prod_{k=3}^{K_b} [ \alpha + (k-1) \beta]  \prod_{c_1=i-1}^{n^b_1-1} (c_1-\beta)  \prod_{l=2}^{K_b} \prod_{c_l=1}^{n_l^b-1} (c_l-\beta)}{\prod_{j=i+1}^{m(b)} (j-1+\alpha) } \\ 
& = \frac{\prod_{k=3}^{K_b} [ \alpha + (k-1) \beta]  \prod_{l=1}^{K_b} \prod_{c_l=1}^{n_l^b-1} (c_l-\beta)}
{\prod_{j=i+1}^{m(b)} (j-1+\alpha) \prod_{c_1=1}^i (c_1-\beta) }
\\ 
& = \frac{\prod_{k=3}^{K_b} [ \alpha + (k-1) \beta] \Gamma(i+\alpha) \prod_{l=1}^{K_b} \Gamma(n_l^b-\beta)}
{ \Gamma(m(b)+\alpha) \Gamma(i-1+\beta)  }
\end{align*}

\bibliographystyle{apa}
\bibliography{main}

\end{document}